\documentclass{article}

\usepackage{microtype}
\usepackage{graphicx}
\usepackage{subfigure}
\usepackage{booktabs} %
\usepackage{makecell}

\usepackage{hyperref}
 \usepackage{booktabs}
\usepackage{multirow} %

\usepackage{glossaries}
\glsdisablehyper
\newacronym{cnn}{CNN}{convolutional neural network}
\newacronym{mlp}{MLP}{multi-layer perceptron}
\newacronym{ntk}{NTK}{neural tangent kernel}
\newacronym{ode}{ODE}{ordinary differential equation}
\newacronym{relu}{ReLU}{rectified linear unit}
\newacronym{lth}{LTH}{Lottery Ticket Hypothesis}
\newacronym{ntt}{NTT}{neural tangent transfer}

\usepackage{bm} 
\usepackage{amsmath,amsfonts,amssymb,mathtools,amsthm}

\newcommand{\RR}{\mathbb{R}}

\newcommand{\mb}{\bm{m}}
\newcommand{\xb}{\bm{x}}

\newcommand{\Ib}{\bm{I}}
\newcommand{\Hb}{\bm{H}}
\newcommand{\yb}{\bm{y}}
\newcommand{\ab}{\bm{a}}
\newcommand{\Xcal}{\bm{\mathcal{X}}}

\renewcommand{\d}{\mathrm{d}}

\newcommand{\thetab}{\bm{\theta}}
\newcommand{\flin}{f^{\text{lin}}}

\theoremstyle{plain}

\newtheorem{prop}{Proposition}

\usepackage[accepted]{icml2020}

\icmltitlerunning{Neural Tangent Transfer}

\begin{document}

\twocolumn[
\icmltitle{Finding trainable sparse networks through Neural Tangent Transfer}

\begin{icmlauthorlist}
\icmlauthor{Tianlin Liu}{fmi,unibas}
\icmlauthor{Friedemann Zenke}{fmi}
\end{icmlauthorlist}

\icmlaffiliation{fmi}{Friedrich Miescher Institute for Biomedical Research, Basel, Switzerland.}
\icmlaffiliation{unibas}{University of Basel, Basel, Switzerland}
\icmlcorrespondingauthor{Friedemann Zenke}{friedemann.zenke@fmi.ch}
\icmlkeywords{Pruning, Sparse neural network}

\vskip 0.3in
]

\printAffiliationsAndNotice{}  
\begin{abstract}
Deep neural networks have dramatically transformed machine learning, but their memory and energy demands are substantial.  The requirements of real biological neural networks are rather modest in comparison, and one feature that might underlie this austerity is their sparse connectivity.   In deep learning, trainable sparse networks that perform well on a specific task are usually constructed using label-dependent pruning criteria.  In this article, we introduce Neural Tangent Transfer, a method that instead finds trainable sparse networks in a label-free manner. 
Specifically, we find sparse networks whose training dynamics, as characterized by the neural tangent kernel, mimic those of dense networks in function space. Finally, we evaluate our label-agnostic approach on several standard classification tasks and show that the resulting sparse networks achieve higher classification performance while converging faster.
\end{abstract}

\section{Introduction}

Deep neural networks achieve human-level performance in a variety of domains \citep{LeCun2015,Schmidhuber2015,silver_mastering_2017}. 
Unlike biological neural networks, however, deep learning systems require extensive computational resources and energy. 
This demand poses a major impediment for future applications in which deep networks are embedded in always-on edge devices and in the Internet of Things (IoT) \citep{Neftci2018data,Roy2019towards,lecun_deep_2019}.

One approach to make artificial neural networks more energy-efficient is to exploit sparseness at both the activity and connectivity level \citep{gong_compressing_2014, lebedev_speeding-up_2015, jaderberg_speeding_2014}. This approach echoes the sparseness principle of biological neural systems, which underlies their superior energy efficiency \citep{Sterling2017principles}. 
To this date, most of the studies on sparse artificial neural networks have focused on \emph{post hoc} pruning, whereby sparse networks are derived from \emph{trained} dense networks. 
This approach, however, is computationally costly 
because it does not allow us to benefit from sparseness during training. 

To tackle this issue, recent work \citep{Lee2018snip, Wang2020picking} proposed identifying trainable sparse networks by pruning dense networks \emph{before} training them. Such \emph{foresight pruning} results in sparse networks that can learn quickly and generalize well in subsequent supervised learning tasks despite having only a few adjustable, nonzero parameters. 
Although they perform well in many scenarios, existing foresight pruning methods can be improved in two ways. First, current methods solely rely on labeled data to determine which parameters to prune. But since labeled data are often scarce \citep{Xu2019uncoupled, Rigollet2019uncoupled}, label-free pruning methods that leverage abundant unlabeled data could offer a decisive advantage. 
Second, existing foresight pruning approaches have focused on global pruning criteria, which are known to remove weights primarily from the fully connected layers of \glspl{cnn} while preserving most parameters in the convolutional layers. Yet, because convolutional layers are responsible for most of the computational burden \citep{Yang2017designing}, global pruning results in comparatively small savings. 
Here, pruning methods that effectively sparsify convolutional filters in a layerwise manner could offer substantial performance improvements.

In this article, we introduce \gls{ntt}, a foresight pruning method that works well in the layerwise setting and without labeled data. 
To the best of our knowledge, this is the first work to find trainable sparse networks in a \emph{label-free} manner. 
Importantly, we show that our method reliably finds trainable \glspl{cnn} with sparse convolutional filters, a situation in which existing foresight pruning methods struggle. 

\section{Prior work \label{sec:prior-work}}
The problem of instantiating sparse neural networks has been considered in several previous studies. 
They can be broadly divided into three categories. 
First, \emph{post hoc pruning}  approaches rely on removing redundant parameters from trained, dense networks.
Second, \emph{on-the-fly pruning} methods enforce sparsity constraints during supervised training.
Finally, \emph{foresight pruning} refers to directly finding a sparse network from scratch that can be trained later to high accuracy. 

\textbf{Post hoc pruning} approaches operate on trained neural networks and attempt to remove network parameters or units that contribute marginally to task performance using different pruning criteria.
Established pruning criteria include parameter magnitude \citep{Jose1988comparing,Stroem1997sparse} 
and the Hessian of the loss function with respect to the weights \citep{LeCun1990, Hassibi1994optimal}. 
Finally, it is common practice to fine-tune or iteratively re-train pruned models, to improve their final performance \citep{Han2015learning,Guo2016dynamic,Zhu2018to}. 

\textbf{On-the-fly pruning} is an alternative approach in which network sparsity is enforced during supervised training. 
This can be achieved in several ways. 
First, by introducing a sparsity-inducing term in the penalty function of supervised training \citep{Chauvin1989, Collins2014MemoryBD, Molchanov2017variational, CarreiraPerpinan2018, Louizos2018learning}. 
Such penalty terms encourage parameters to be close to zero. 
Second, by use of dynamic network rewiring rules over training time that keeps network sparsity below a given threshold \citep{Mocanu2018,Bellec2018deep,Yan2019efficient, Evci2020rigging}.

\textbf{Foresight pruning} refers an approach in which one first prunes a network at its initialization and then trains the pruned network to convergence. In part, foresight pruning is motivated by evidence that specific sparse networks can be trained to yield comparable performance to the corresponding dense model \citep{Liu2019rethinking,Frankle2019,Morcos2019one}. 
To that end, \citet{Lee2018snip, Lee2020signal} attempted to identify trainable sparse network structure at the network initialization stage based on the connection sensitivity criterion. \citet{Wang2020picking} proposed to find sparse networks that preserve the error gradients after pruning

\section{Neural Tangent Transfer framework}

All previous approaches have relied on labeled data to find trainable sparse neural networks.
In this article, we develop a label-free approach to find sparse networks whose training evolution in the function space are comparable to those of their corresponding dense counterparts. The resulting sparse networks exhibit higher performance when trained on subsequent supervised tasks.
We approach this problem by considering the evolution of a neural network's output from a given starting point, determined by its random initialization, to an endpoint, characterized by the trained network parameters.
To yield good performance, the endpoint of the sparse network's output evolution needs to lie close to the output of a well-performing network. 
However, in supervised learning, the endpoint itself depends on labeled data.
The central idea of our approach is to find sparse networks that share the \emph{same} starting point in function space as a corresponding dense network 
and whose network output evolution during supervised training is characterized by similar dynamics.
If both criteria are satisfied, also the endpoints will be closely matched, thus ensuring the trainability of the sparse network (Fig.~\ref{fig:schematic-outputspace}). 
To influence the training dynamics of a sparse network \emph{before} training it, we leverage recent theoretical insights pertaining to the \gls{ntk} \citep{Jacot2018neural, Arora2019On, Lee2019wide, Chizat2019lazy}, which allow anticipating the training dynamics in a label-free manner, 
and devise an efficient algorithm that exploits this knowledge to constrain the sparse network's output evolution to our advantage.

\begin{figure}[tbhp] 
\center
    \includegraphics[width=0.9\columnwidth]{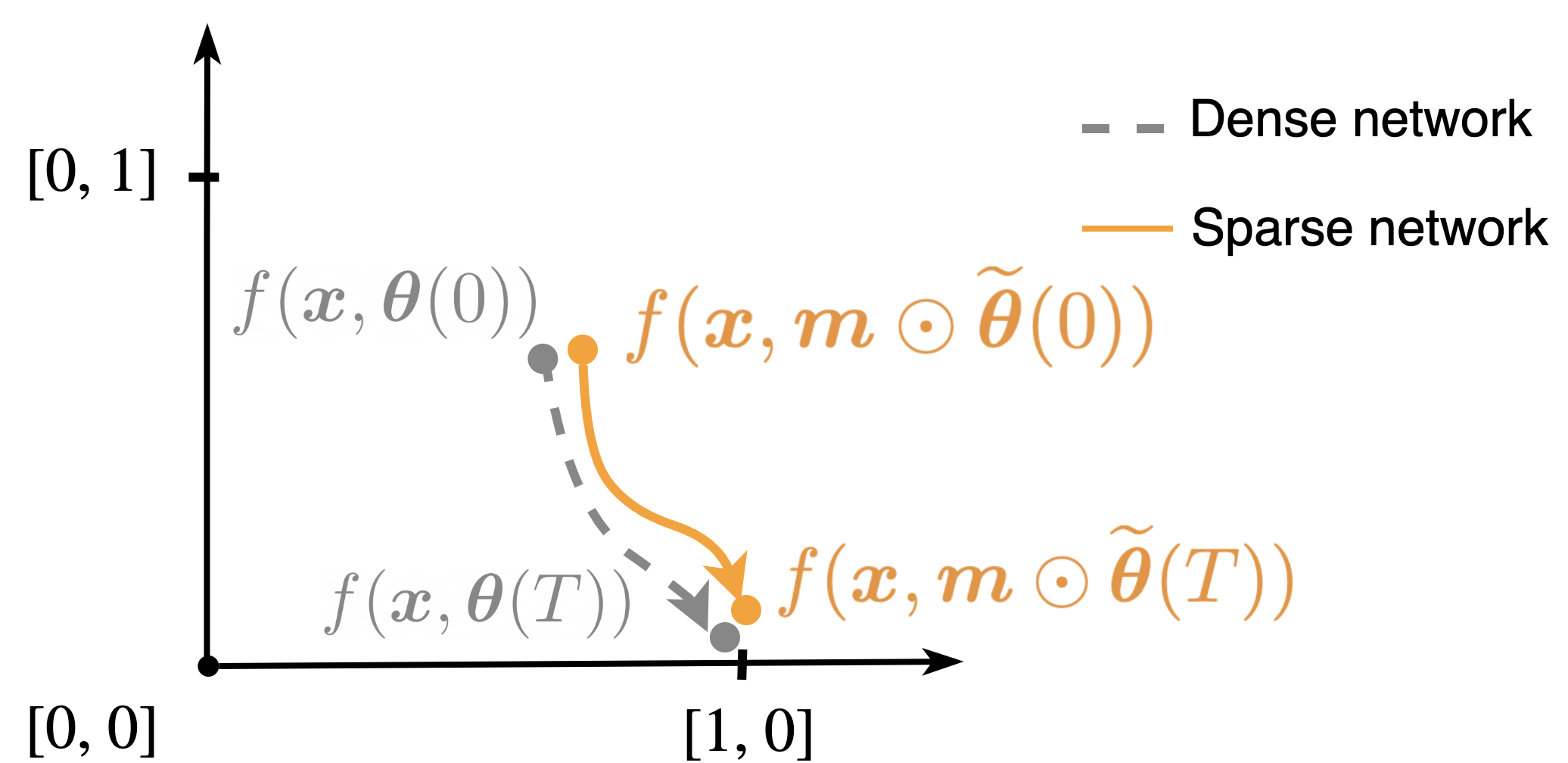}
    \caption{Schematic illustration of neural networks' output evolution during supervised training from time $t = 0$ (starting point) to $t = T$ (endpoint). 
    Here, $[1,0]$ and $[0, 1]$ are one-hot targets for a binary classification task; 
    the input data $\xb$ comes from the class whose correct label is $[1,0]$. The dashed grey curve shows to the output evolution of a dense network, which, during training, moves toward the correct target $[1,0]$. 
    To ensure trainability, we aim to find a sparse network whose starting point and output evolution (orange, solid curve) during subsequent training closely follows the dense one's, 
    such that both output evolution terminate at similar endpoints. \label{fig:schematic-outputspace}}
\end{figure}

More formally, consider a densely parameterized feed-forward neural network $f(\xb, \thetab)$ with input data $\xb$ and trainable parameters $\thetab$. 
During supervised learning, its output evolution is given by $\{f(\xb, \thetab (t)) \}_{t \geq 0}$ with $t$ being the training time (e.g., iteration number) and $\thetab (t)$ the parameters at that time. Similarly, consider training a sparse network $f(\xb, \mb \odot \widetilde{\thetab})$, where $\widetilde{\thetab}$ are the trainable parameters, 
$\mb$ is a fixed sparsity-inducing mask, and $\odot$ is the Hadamard product. The output evolution of the sparse network during training is described by $\{f(\xb, \mb \odot \widetilde{\thetab} (t)) \}_{t \geq 0}$. For the sparse network to yield good performance, we aim to select a good combination of $\mb$ and $\widetilde{\thetab} (0)$ to ensure that its output evolution approximately follows the dense network's evolution throughout training, i.e., 
\begin{equation} \label{eqn:dense-sparse-net-similar-dynamics}
f(\xb, \mb \odot  \widetilde{\thetab} (t)) \approx f(\xb, \thetab (t)),
\end{equation}
for all training inputs $\xb$ and training time $t \geq 0$ (cf.\ Fig.~\ref{fig:schematic-outputspace}).

With this goal in mind, we proceed in the following steps.
First, we characterize both the dense and sparse networks' output evolution  (Section~\ref{subsec:output-dynamics}) during supervised learning using the tool of \glspl{ntk}.
Second, we motivate an objective function aimed at minimizing the difference between the two respective evolution (Section~\ref{subsec:ntt_objective}). 
Finally, we introduce an algorithm that minimizes this objective function and, by tuning the sparse network to have similar training dynamics to a corresponding  dense network, finds a trainable sparse network in a label-free manner (Section~\ref{subsection:ntt_implementation}).

\subsection{Label-free characterization of output evolution for dense and sparse networks \label{subsec:output-dynamics}}

In this subsection, we aim to characterize the output evolution of dense and sparse neural networks during training using analytical insights and tools developed in the context of linearized networks and \gls{ntk} \citep{Jacot2018neural, Arora2019On, Lee2019wide, Chizat2019lazy}. 

We first focus on a dense neural network $f(\cdot, \thetab): \RR^d \to \RR$ for $d$ being the dimension of input and $\thetab$ the vector of all parameters. Examples of such dense networks include \glspl{mlp} with fully connected layers and \glspl{cnn} with dense filters. For ease of notation, we assume that the network has a scalar-valued output, but the framework can easily be extended to the vector-valued case. Given a training dataset $\{\xb_i, y_i\}_{i = 1}^n \subset \RR^d \times \RR$ of $n$ input-target pairs, we consider the empirical risk minimization problem with the quadratic loss function 
\begin{eqnarray} 
 L (\thetab) &=& \frac{1}{2} \sum_{i = 1}^n \Big ( f(\xb_{i}, \thetab) - y_{i} \Big )^2\\
 & = & \frac{1}{2} \|f(\Xcal, \thetab) - \yb \|_2^2.
 \label{eqn:quadratic-loss}
 \end{eqnarray}
where for ease of notation we have written $\Xcal = \{\xb_i\}_{i = 1}^n$ as the collection of training inputs, $f(\Xcal, \thetab )=  \big( f (\xb_i, \thetab) \big)_{i\in[n]} \in \RR^n $ as the concatenation of network outputs of input data $\xb_i$, and $\yb = (y_i)_{i \in [n]} \in \RR^n$ as the corresponding targets. 

We now consider training the network with continuous-time gradient descent as characterized by the gradient flow 
\begin{equation} \label{eqn:gradient-flow}
\frac{\d \thetab (t)}{ \d t}   = - \nabla_{\thetab} L  \big (\thetab (t)   \big ),
 \end{equation}
for training time $t \geq 0$.

From Eqns.~\eqref{eqn:quadratic-loss} and \eqref{eqn:gradient-flow}, and by applying the chain rule follows the network's output evolution during training 
\begin{equation} \label{eqn:output-dynamics}
\frac{\d f(\Xcal, \thetab (t) ) }{ \d t}  = - \Hb_{\thetab (t)} \Big [ f \Big ( \Xcal, \thetab (t) \Big ) - \yb \Big ], 
\end{equation}
where $\Hb_{\thetab (t)}$ is an $n \times n$ positive semidefinite matrix whose $(i, j)$-th entry is the value of the inner product $\Big \langle \nabla_{\thetab} f \big(\xb_i, \thetab(t) \big), \nabla_{\thetab} f \big(\xb_j, \thetab(t) \big) \Big \rangle$. At initialization $t = 0$, the inner product function \small{\begin{alignat}{2}
K_{\thetab(0)} ( \cdot, \cdot): \RR^d & \times & \RR^d & \to \RR, \nonumber \\
(\xb &,& \xb') & \mapsto  \Big \langle \nabla_{\thetab} f(\xb, \thetab(0)),  \nabla_{\thetab} f(\xb', \thetab(0)) \Big \rangle
\end{alignat}} \normalsize
is called the empirical neural tangent kernel of the network \citep{Jacot2018neural}.

The dynamics of Eqn.~\eqref{eqn:output-dynamics} are difficult to analyze because of the time-varying matrix $\Hb_{\thetab (t)}$. 
To simplify the analysis, previous work has resorted to studying the neural network's linearized approximation around initialization \citep{Lee2019wide,Chizat2019lazy}. Formally, for a neural network model $f( \xb, \thetab)$ with input $\xb$, parameters $\thetab$, and initialization $\thetab(0)$, its linearization $\flin(\xb, \thetab )$ around initial parameters is defined as
\begin{equation} \label{eqn:linearized-network} \small
\flin(\xb, \thetab ) \coloneqq f  \big(\xb, \thetab(0)  \big) + \Big \langle {\thetab} - \thetab(0), \nabla_{\thetab} f \big(\xb, \thetab(0) \big) \Big \rangle. 
\end{equation} \normalsize

Let $ \flin (\Xcal, \thetab (t)) = \Big ( \flin (\xb_i, \thetab (t)) \Big )_{i\in[n]} \in \RR^n $ be the concatenation of linearized network outputs of all training samples $\xb_i$ at time $t$. The evolution of the linearized network's output is described by the following first order \gls{ode}
\begin{equation} \label{eqn:linearized-network-output-dynamics}
\frac{\d \flin (\Xcal, \thetab (t))}{\d t}  = - \Hb_{\thetab(0)} \Big [ \flin \Big (\Xcal, \thetab (t) \Big ) - \yb \Big ].
\end{equation}
Note that Eqn.~(\ref{eqn:linearized-network-output-dynamics}) represents a substantial simplification of Eqn.~\eqref{eqn:output-dynamics} because it only depends on $\Hb_{\thetab(0)}$, which is constant during training and fully characterized by the network's initialization. The solution of the corresponding system of linear \glspl{ode} is known analytically
\begin{equation} \label{eqn:linear-dynamics-solved} \small
\flin \Big (\Xcal, \thetab (t) \Big ) = e^{- t \Hb_{\thetab(0)}} \flin \Big (\Xcal, \thetab(0) \Big ) + \left [ \Ib - e^{-t \Hb_{\thetab(0)}} \right] \yb, 
\end{equation} \normalsize
for $t \geq 0$. 
While this linear approximation of a network's output evolution becomes exact as the width of the neural network goes towards infinity \citep{Jacot2018neural},
empirically it is known that the linear approximation is quite accurate even for finite-width networks \citep{Lee2019wide}. 

So far we have considered the output evolution of dense networks and their linearized approximations. 
We now turn to the analysis of neural networks with a fixed sparsity-inducing mask. In accordance with Eqn.~\eqref{eqn:linearized-network}, we define the linearized sparse network with input $\xb$, mask $\mb$, trainable parameters $\widetilde{\thetab}$, and initialization $\widetilde{\thetab} (0)$ as 
\begin{equation} \label{eqn:linearized-sparse-network}   \small
\begin{split}
& \flin(\xb, \mb \odot \widetilde{\thetab} ) \\
& \quad \coloneqq f  \Big(\xb, \mb \odot \widetilde{\thetab} (0)  \Big) + \Big \langle \widetilde{\thetab} - \widetilde{\thetab} (0), \nabla_{\widetilde{\thetab}} f \Big(\xb, \mb \odot \widetilde{\thetab} (0) \Big) \Big \rangle. 
\end{split}
\end{equation} \normalsize

In line with Eqn.~\eqref{eqn:linear-dynamics-solved}, the output evolution of linearized sparse sparse neural network in Eqn.~\eqref{eqn:linearized-sparse-network} admits the analytic solution
\begin{equation} \label{eqn:linear-dynamics-solved-sparse} \small
\begin{split}
& \flin (\Xcal, \mb \odot \widetilde{\thetab} (t))\\ 
& \quad = e^{- t \Hb_{\mb \odot \widetilde{\thetab} (0)}} \flin \Big (\Xcal, \mb \odot \widetilde{\thetab} (0) \Big ) + [ \Ib - e^{-t \Hb_{\mb \odot \widetilde{\thetab} (0)}}] \yb,
\end{split}
\end{equation} \normalsize
where $\Hb_{\mb \odot \widetilde{\thetab}(0) }$ is an $n \times n$ positive semidefinite matrix whose $(i, j)$-th entry is \mbox{$\langle \nabla_{\widetilde{\thetab}} f \big(\xb_i, \mb \odot \widetilde{\thetab} (0) \big  )  , \nabla_{\widetilde{\thetab}} f \big  (\xb_j, \mb \odot \widetilde{\thetab} (0) \big  ) \rangle$}.

\subsection{The neural tangent transfer (NTT) objective function \label{subsec:ntt_objective}}

After having characterized the training dynamics of dense and sparse networks, we proceed to use this knowledge in a label-free manner to find sparse neural networks whose training evolution will remain close to the corresponding dense network's evolution during subsequent training. 
In other words, for a given instance of dense network at initialization $f(\Xcal, \thetab (0))$, we wish to select $\mb$ and $\widetilde{\thetab} (0)$ that minimize the distance
\begin{equation} \label{eq:standard-nn-optimization}
 \sum_{t = t_0}^{t_T} \|f (\Xcal, \mb \odot \widetilde{\thetab} (t)) - f(\Xcal, \thetab (t)) \|_2^2
\end{equation}
during training in discrete time $0 \leq t_0 < \cdots < t_T$. Unfortunately, it is impossible to evaluate Eqn.~\eqref{eq:standard-nn-optimization} without labels because the trained parameters $\{\widetilde{\thetab} (t)\}_{t > 0}$ are the result of the label-dependent supervised learning procedure. We therefore proceed in two steps. First, we make use of the linear approximation motivated in the preceding subsection. Second, we derive an auxiliary objective that is sufficient to minimize our objective derived in the first step, but in a label-free manner.

We start by replacing the corresponding output evolution in  Eqn.~\eqref{eq:standard-nn-optimization} by its linear approximations
\begin{equation} \label{eq:linear-nn-optimization}
 \sum_{t = t_0}^{t_T} \|\flin (\Xcal, \mb \odot \widetilde{\thetab} (t)) - \flin(\Xcal, \thetab (t)) \|_2^2.
\end{equation}
Note that Eqn.~\eqref{eq:linear-nn-optimization} is still label-dependent due to its dependence on the trained parameters  $\{\widetilde{\thetab} (t)\}_{t > 0}$. 
Yet, by comparing the linearized network's outputs given by Eqn.~\eqref{eqn:linear-dynamics-solved} and Eqn.~\eqref{eqn:linear-dynamics-solved-sparse}, we see that a sufficient condition for minimizing the quantity in Eqn.~\eqref{eq:linear-nn-optimization} is to minimize the following label-free objective:
\begin{equation} \small \label{eq:ntt_objective}
\begin{split}
J_{\thetab(0)} \Big  (\mb \odot \widetilde{\thetab} (0) \Big ) &  = \frac{1}{n} \Big \| f \Big (\Xcal, \mb \odot \widetilde{\thetab} (0) \Big ) - f \Big (\Xcal, \thetab(0) \Big ) \Big\|_2^2 \\
& \quad + \frac{\gamma^2}{n^2} \Big \| \Hb_{\mb \odot \widetilde{\thetab} (0)} - \Hb_{\thetab(0)} \Big \|_{F}^2,
\end{split}
\end{equation}\normalsize
where $\gamma^2 > 0$ is a strength parameter and $\| \cdot \|_{F}$ denotes the Frobenius norm of a matrix. 
We note that the objective function $J_{\thetab(0)}$ contains two terms. The first term $\frac{1}{n} \| f (\Xcal, \mb \odot \widetilde{\thetab} (0)) - f (\Xcal, \thetab(0))\|_2^2$ measures the distance between the dense and the sparse network's starting point (cf. Fig. \ref{fig:schematic-outputspace}).
The second term $\frac{\gamma^2}{n^2} \| \Hb_{\mb \odot \widetilde{\thetab} (0)} - \Hb_{\thetab(0)} \|_{F}^2$ measures the discrepancy between the empirical \glspl{ntk} of the dense and sparse network. 
When minimized, the latter term encourages the output evolution of a sparse network to remain close to its dense counterpart during training.  
We call the optimization process that minimizes $J_{\thetab(0)} (\mb \odot \widetilde{\thetab} (0))$ \emph{neural tangent transfer} (NTT). As the NTT objective function depends on the dense network's initialization $\thetab(0)$, for convenience, we sometimes refer to the dense network $f(\cdot, \thetab(0))$ as the NTT teacher and the sparse network $f(\cdot, \mb \odot \widetilde{\thetab} (0))$ as the NTT student. 

Since the NTT objective function is based on a linear approximation of a neural network model, we are prompted to ask what we can expect if the approximation is exact. Proposition~\ref{prop:linear-opt} states that when the NTT teacher and student models are linear, the sparse student model that minimizes the NTT objective function possesses the same trainability as the dense teacher. That means, upon supervised training, the sparse student's output evolution matches its dense teacher's evolution exactly. 

\begin{prop} \label{prop:linear-opt}
Consider a linear NTT teacher model $f(\xb, \ab(0)) = \ab(0)^\top \xb$, where $\xb$ and $\ab(0) \in \RR^d$ are model input and initial parameters. Suppose that we are given a linear NTT student model $g_{\mb}(\xb, \widetilde{\ab}(0)) = (\mb \odot \widetilde{\ab}(0))^{\top} \xb$ whose initial parameters $\widetilde{\ab}(0)$ are NTT-optimal in the sense that $J_{\ab(0)}(\mb \odot \widetilde{\ab} (0)) = 0$. Then upon continuous-time and quadratic-loss based gradient descent training, the dense and sparse models' outputs evolve in the same way:
\[ f(\xb, \ab(t)) = g_{\mb}(\xb, \widetilde{\ab} (t)), \]
for all training inputs $\xb$ and time steps $t \geq 0$.
\end{prop}

The proof to the Proposition \ref{prop:linear-opt} uses simple calculus and can be found in Appendix \ref{proof-to-proposition}. 

\subsection{Algorithmic implementation of NTT \label{subsection:ntt_implementation}}

The NTT objective function (Eqn.~\eqref{eq:ntt_objective}) can readily be optimized using stochastic gradient descent (SGD) techniques. Specifically, we take following procedural steps:
\begin{enumerate}
\item \textbf{Choose a network architecture.} This includes fixing the number of layers, number of units per layer, etc. 
\item \textbf{Instantiate the NTT teacher network.} Randomly initialize a dense network $f(\cdot, \thetab)$ according to an initialization scheme (e.g., Glorot initialization \citep{Glorot10}). 
\item \textbf{Initialize the NTT student network.} For a desired sparsity level of the network $0 < p < 1$, remove $1 - p$ fraction of the weight parameters of $\thetab$ having the smallest magnitudes to create a binary mask $\mb$. The NTT student is initialized to be $f(\cdot, \mb \odot \widetilde{\thetab})$, where $\widetilde{\thetab} =  \thetab$. 
\item \textbf{Set the hyperparameters for SGD.} Choose SGD batch size, total number of iterations, optimizer, learning rate, the strength parameter $\gamma^2$ in the NTT objective function of Eqn.~\eqref{eq:ntt_objective}, and a \emph{mask update frequency} (see the next step below).
\item \textbf{Adjust NTT student network parameters and masks via SGD.} To update the parameters, we fix a mask $\mb$ and repeatedly draw minibatches of label-free data samples to perform SGD with weight decay:
\begin{equation} \label{eq:ntt_update} 
 \widetilde{\thetab} \leftarrow \widetilde{\thetab} - \eta \cdot \nabla_{\widetilde{\thetab}} J_{\thetab} (\mb \odot \widetilde{\thetab}) - \beta \cdot \mb \odot \widetilde{\thetab}  ,
\end{equation}
where $\eta > 0$ is the learning rate and $\beta$ is a small weight-decay constant that helps the learned parameters generalize to the unseen data. Here, the back-propagated gradients flow through the fixed mask $\mb$ and thus the masked-out parameters are not updated. Every a few steps, as specified by \textit{the mask update frequency}, we update the mask based on the current weight magnitudes. 
\item \textbf{Return the optimized NTT student} Upon completion of SGD, we are left with an optimized NTT student network $f(\cdot, \mb \odot \widetilde{\thetab} (0))$.
\end{enumerate}

For all experiments in this article, we used the procedure outlined above. The specific experimental setup and all hyper-parameter choices are provided in the corresponding experiment sections and Appendix \ref{supp:experiment_setup}. Below we offer some general remarks pertaining to the implementation. 
Neural networks for classification tasks are usually trained using a softmax activation defined on the logits at the output. However, the softmax activation contains no trainable parameters. Therefore, we used the logits output directly for NTT. On subsequent supervised learning tasks, in which labels are present, we equipped the last layer with a softmax activation unless mentioned otherwise. 
With slight modifications, we applied the outlined steps to produce either layerwise or globally sparse networks (see Sec.\ \ref{sec:experiment} for details). In either scheme, we only pruned weights and no bias terms.

\section{Experiments \label{sec:experiment}}

We numerically evaluated the trainability of sparse networks found by the \gls{ntt} algorithm for both \glspl{mlp} and \glspl{cnn}
on standard datasets, including MNIST \citep{Lecun1998}, Fashion MNIST \citep{xiao2017fashion}, CIFAR-10 \citep{Krizhevsky09}, and SVHN \citep{Netzer2011reading}. All experiments were performed in JAX \citep{Jax2018github} together with the neural-tangent library \citep{neuraltangents2020}. The code to reproduce our experiments is available at \url{https://github.com/fmi-basel/neural-tangent-transfer}.

\paragraph{General experiment procedure.}
We adopted the following procedure to evaluate the trainability of the sparse network yielded by NTT. First, we carry out NTT only using label-free data to learn the sparse network initialization (\emph{NTT initialization}). 
To evaluate the learned results, we presented the labels in subsequent supervised learning tasks and used the NTT initialization to solve the task. 

\paragraph{Baselines.}
We compared NTT against the following baselines:
\begin{itemize}
    \item Randomly sampled sparse networks from densely initialized neural networks.  
    \item Variance-scaled random sparse networks derived as a modification of Kaiming initialization \citep{He2014delving}: For $w$ being a random variable taking values as weights in a layer, we set $\text{Var}(w) = 2/(n^{{\text{in}}} \cdot p)$ prior to pruning, where $\text{Var}$ is the variance, $n^{\text{in}}$ is the number of input units (fan-in) to that layer and $p$ is fraction of remaining weights. 
    \item Sparse networks pruned with SNIP \citep{Lee2018snip}, a label-dependent foresight pruning method. Note that SNIP was originally formulated for global pruning; straightforwardly, we extend SNIP to layerwise pruning and refer to this variant as Layerwise-SNIP (see Appendix \ref{supp:snip_logit_snip}).
    
    \item Sparse networks pruned with Logit-SNIP, a label-free variant of SNIP obtained by replacing the supervised loss function that occurs in SNIP's pruning objective with the sum of squares of the network's logits (for details, see Appendix \ref{supp:snip_logit_snip}).
\end{itemize}

\paragraph{Pruning schemes.}
We considered both \emph{layerwise} and \emph{global} pruning schemes for the proposed and the baseline methods. 
In the layerwise scheme, we removed a fixed fraction of the parameters in each layer; in the global scheme, we pruned weights in all layers collectively without enforcing the weight-sparsity level of any specific layer. 
In the main text of this article, we primarily focused on layerwise pruning (Sec. \ref{subsec:probing} and \ref{subsec:layerwise}). However, we devote Section~\ref{subsec:global} to a case study of global pruning methods and defer a more detailed study of global pruning to the Appendix.

\subsection{Visualizing output dynamics of NTT initialization \label{subsec:probing}} 

We first sought to confirm our intuitions underlying the formulation of NTT (cf.\ Fig.~\ref{fig:schematic-outputspace}). 
To that end, we devised a simple toy problem in which we are able to visually compare output evolution of sparse and dense networks. 
Recall that the goal of the NTT is to learn a well-initialized sparse network (NTT student) so that its output dynamics under supervised training mimics the dynamics of its dense counterpart (NTT teacher), as motivated by Eqn.~\eqref{eqn:dense-sparse-net-similar-dynamics} and Fig.~\ref{fig:schematic-outputspace}. 
To see whether this goal is empirically fulfilled, we trained both dense and sparse neural network models on a binary classification task using a subset of the MNIST dataset of handwritten digits (0 and 1). Since this is a binary classification task, the network outputs at each training time are two-dimensional vectors and therefore can be visualized.

\begin{figure}[htbp] 
    \includegraphics[width= 0.45 \textwidth]{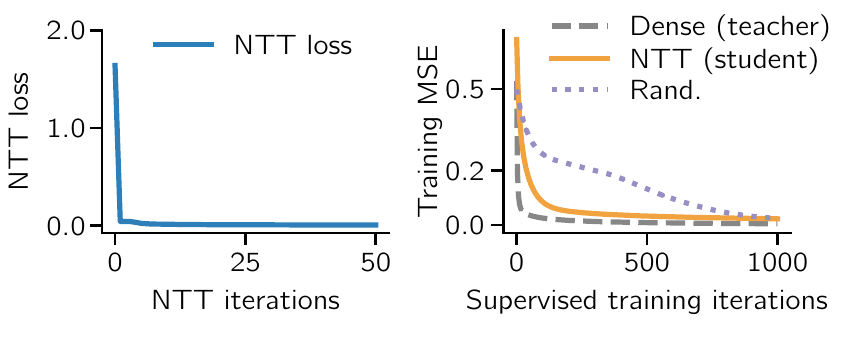}
          \vspace{-0.1 in}
    \caption{\textbf{NTT loss and supervised training loss.} Using the NTT algorithm (Sec.~\ref{subsection:ntt_implementation}), we randomly initialized a dense network (NTT teacher) and used it to derive a sparse network (NTT student). The NTT loss during this process is shown in the left panel. We then used the NTT teacher (dense), the NTT student (sparse), and a random sparse network as initialization to solve the binary digit classification task (right panel). The dense NTT teacher (grey, dashed curve) solved the task with ease. The NTT student (orange curve) solved the task slower than the NTT teacher, but much faster than the random sparse network (purple, dotted curve).   \label{fig:losses_mlp_probing}}
      \vspace{-0.1 in}
\end{figure}

Concretely, we used an \gls{mlp} with two hidden layers containing 300 and 100 units with \glspl{relu} followed with 2~linear output units. We randomly initialized the dense network using the Glorot initialization scheme \citep{Glorot10} and designate it as the NTT teacher. We then perform NTT using the algorithm outlined in subsection \ref{subsection:ntt_implementation} to train a 90\% sparse NTT student network (see Appendix \ref{supp:visualization} for the NTT parameters). During training, the loss of the NTT objective function dropped sharply (Fig.~\ref{fig:losses_mlp_probing} left). 
Up to this point, no data label was used. 

Next, to assess the trainability of the NTT student, we used a supervised label-dependent loss to train the NTT teacher (dense), NTT student (sparse), and a randomly initialized sparse network with 5000~iterations of gradient descent. 
During training, we tracked both the loss and the output evolution of all networks.
All networks solved this simple problem with low error (Fig.~\ref{fig:losses_mlp_probing} right). 
Not surprisingly, the dense NTT teacher network solved this simple task easily, reducing the training error with only a few iterations (dashed, grey curve).
Importantly, the sparse NTT student network (orange curve) minimized the training error much faster than the randomly initialized sparse network (purple, dotted curve).

\begin{figure}[htbp] 
    \includegraphics[width = 0.45 \textwidth]{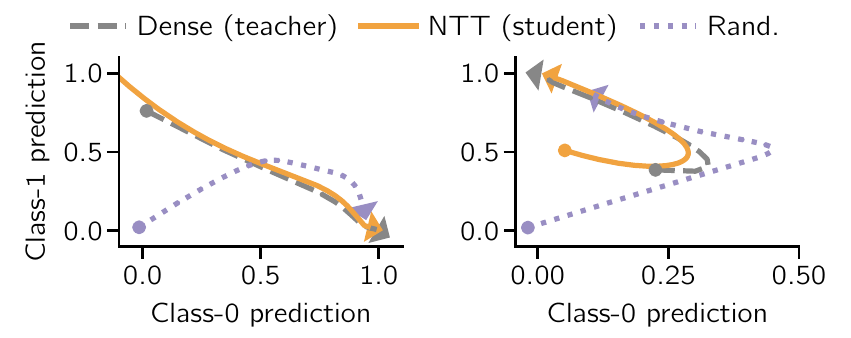}
    \caption{\textbf{The NTT student's output dynamics closely mimic the NTT teacher}. 
Left: output evolution of different networks during supervised training obtained as the averaged prediction of digit-0 images. The arrows indicate the direction of the output evolution. Right: Same as left, but for digit-1 images. 
As expected, we found that the output evolution of the NTT teacher network (dashed grey curve) and the NTT student network (orange curve) closely resemble each other whereas the output evolution of the random sparse control network (dotted purple curve) was notably different.
\label{fig:mlp_probing}}
\vspace{-0.1 in}
\end{figure}

Finally, we visualized the networks' output evolution. Because predictions are two-dimension vectors approximating one-hot encoded binary labels of images, we averaged predictions across images from the same ground-truth class (digit 0 or 1) at each iteration. 
As anticipated, we found that the output evolution of the NTT teacher and student were similar whereas the random sparse initialization was notably different (Fig.~\ref{fig:mlp_probing}).

While the above analysis was conducted using an \gls{mlp}, we have performed a similar experiment with a \gls{cnn} and found qualitatively similar dynamics (Appendix \ref{supp:visualization}). 
These results confirmed our intuition that NTT initialized networks do indeed behave qualitatively similar to their dense counterparts during training.

\subsection{Evaluating layerwise sparse NTT initialization \label{subsec:layerwise}}
Having confirmed out intuition on the toy example, we now proceed to examine our method on larger datasets including MNIST, Fashion MNIST, CIFAR-10, and SVHN. 

\begin{figure}[htb] 
    \centering
    \includegraphics[width = 0.45 \textwidth]{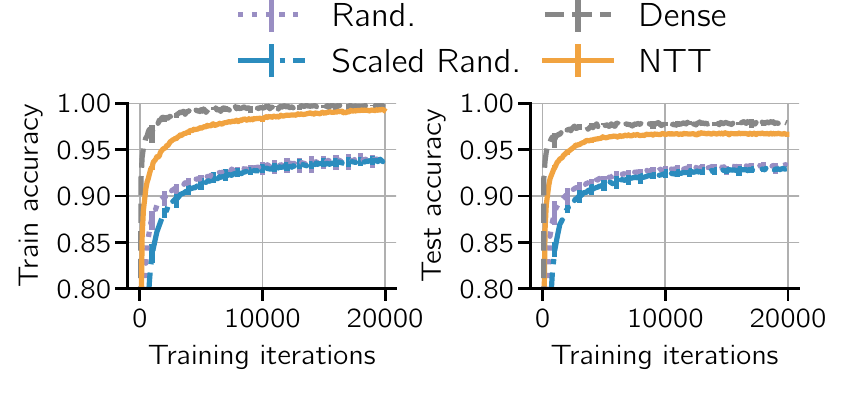}
    \vspace{-0.2 in}
    \caption{\textbf{Learning curves of NTT and baseline networks trained on MNIST.} The curves illustrate training and testing accuracy of Lenet-300-100 networks. All networks except the dense ones (grey, dashed curves) are sparse networks with 3\% nonzero weights. Each curve is the average of five independent trials. 
    Error bars denote the minimum and maximum of any trial. \label{fig:paper_plot_mnist_mlp_lenet}} 
\end{figure}

\begin{figure}[htb] 
    \centering
    \includegraphics[width = 0.45 \textwidth]{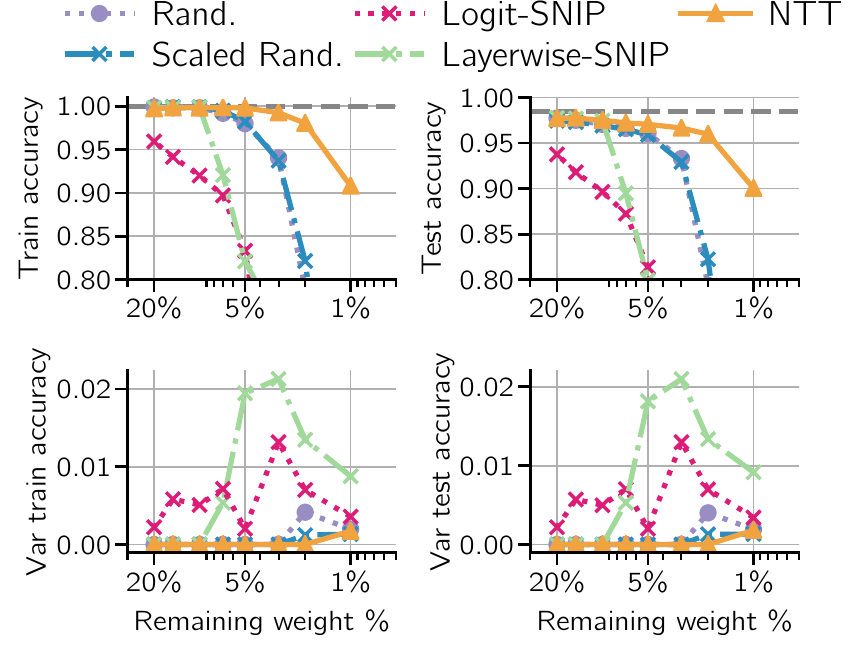}
    \caption{\textbf{Supervised performance on MNIST using layerwise pruned Lenet-300-100.}
    Top row: Classification accuracy on training (left) and testing datasets (right) for different levels of sparseness in networks with NTT initialization and baseline methods.
    Bottom row: Corresponding endpoint variance in accuracy of the data points in the row above computed over 5 repetitions.
    \label{fig:overall_mnist_mlp_lenet}}
    \vspace{-0.2 in}
\end{figure}

\begin{figure}[htb] 
   \centering
    \includegraphics[width = 0.45 \textwidth]{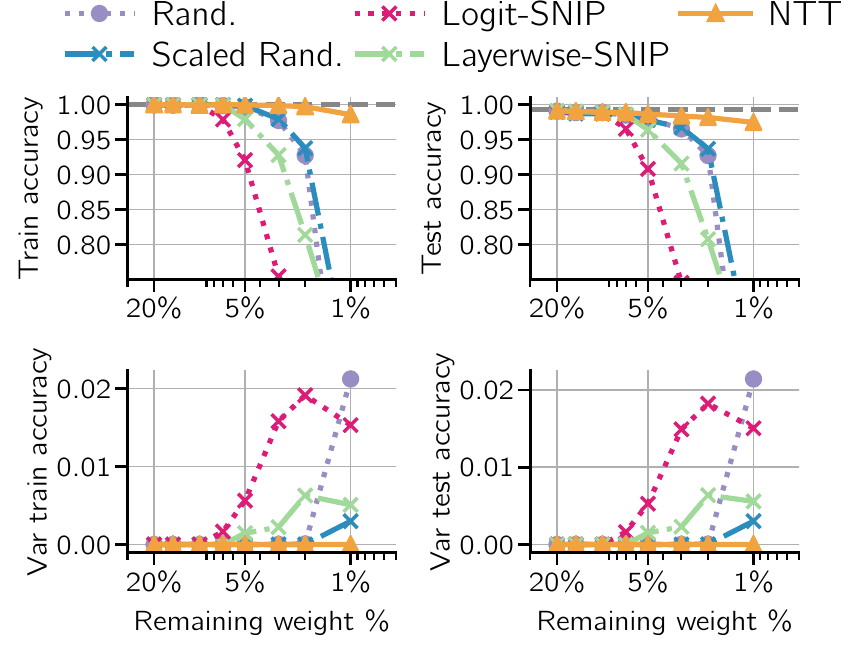}
    \caption{\textbf{Supervised performance on MNIST using layerwise pruned Lenet-5-Caffe.
    } \label{fig:overall_mnist_cnn_lenet}}
    \vspace{-0.2 in}
\end{figure}

\paragraph{MNIST.} 
In the MNIST experiments, we first applied NTT to \glspl{mlp} of the Lenet-300-100 architecture \citep{Lecun1998} at different sparsity levels. We noted that NTT initialized networks of a given sparsity could be trained faster and to higher accuracy than corresponding controls (Fig.~\ref{fig:paper_plot_mnist_mlp_lenet}).

When comparing the classification accuracy at different sparsity levels, NTT initialized networks showed generally higher accuracy (Fig.~\ref{fig:overall_mnist_mlp_lenet} top row).
Moreover, we found that NTT initialization resulted in a lower variance of their final classification accuracy between different repetitions of the same experiment (Fig.~\ref{fig:overall_mnist_mlp_lenet}, bottom row).
These findings illustrate that NTT, as intended, finds sparse networks that can be trained faster, with higher accuracy, and reduced variability between trials.

To test whether these findings would generalize to \glspl{cnn}, we repeated the above experiments using a LeNet-5-Caffe \gls{cnn} architecture, which is a modified variant of the LeNet-5 model proposed by \citet{Lecun1998}.
In good agreement with the \gls{mlp} results, we found that NTT initialized sparse networks showed significantly better performance over baseline methods (Fig.~\ref{fig:overall_mnist_cnn_lenet} top)
and a notable reduction of the variance of their final accuracy (Fig.~\ref{fig:overall_mnist_cnn_lenet} bottom).

\paragraph{Fashion MNIST.} 
We repeated our experiments on Fashion MNIST. Similar to the MNIST task, we used Lenet-300-100 and Lenet-5.
Both \gls{mlp} (LeNet-300-100) and \gls{cnn} (LeNet-5-Caffe) trained on fashion MNIST recapitulated our major findings on the MNIST experiments from above: Compared to baseline methods,
NTT initialized networks converged faster, exhibited higher endpoint accuracy, and showed less variability over different repeats (Figs.~\ref{fig:overall_fashion_mnist_mlp_lenet} and~\ref{fig:overall_fashion_mnist_cnn_lenet}). 

\begin{figure}[htb] 
    \centering
    \includegraphics[width = 0.45 \textwidth]{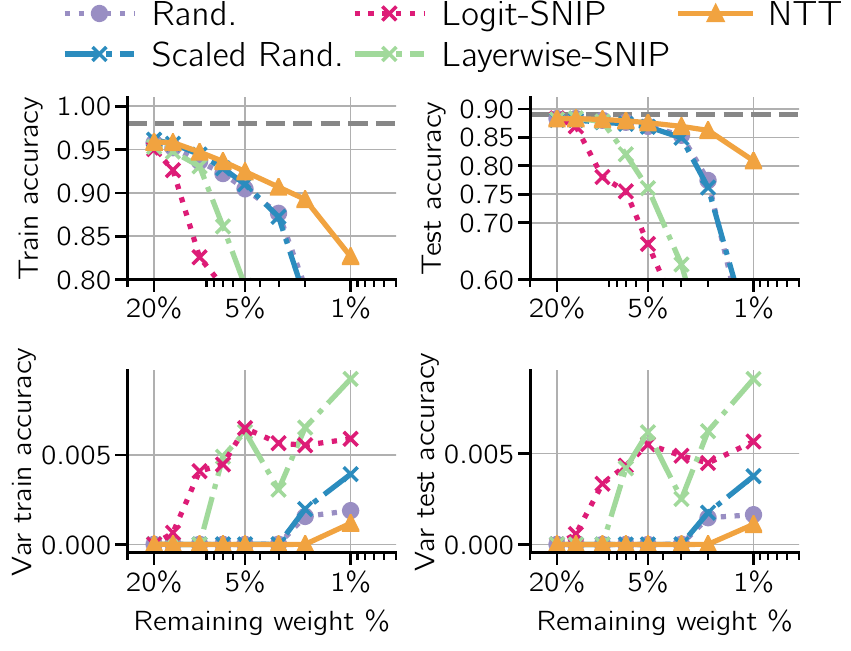}
    \caption{\textbf{Supervised performance on Fashion MNIST using layerwise pruned Lenet-300-100.}} \label{fig:overall_fashion_mnist_mlp_lenet}
    \vspace{-0.1 in}
\end{figure}

\begin{figure}[htb] 
    \centering
    \includegraphics[width = 0.45 \textwidth]{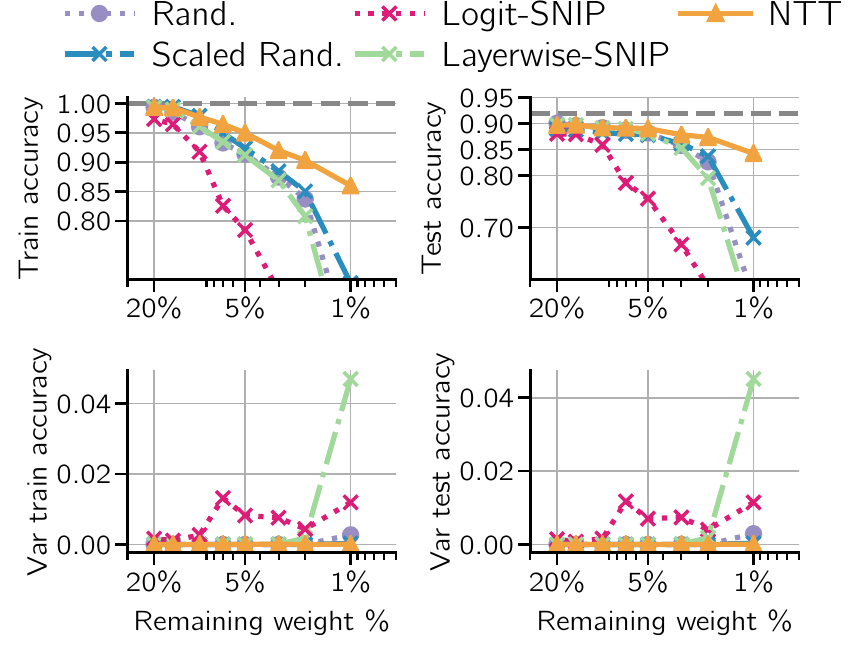}
    \caption{\textbf{Supervised performance on Fashion MNIST using layerwise pruned Lenet-5-Caffe.} \label{fig:overall_fashion_mnist_cnn_lenet}}
    \vspace{-0.1 in}
\end{figure}

\paragraph{CIFAR-10.}
For the CIFAR-10 classification task, we trained a \gls{cnn} model that we called Conv-4. This model has 4 convolution layers followed by 2 feedforward layers with dropout (see Appendix \ref{supp:conv_4} for details), corresponding to a slightly scaled-up version of the default  CNN example from Keras CIFAR-10 classification tutorial\footnote{\url{https://keras.io/examples/cifar10_cnn/}}. 

\begin{figure}[htb] 
    \includegraphics[width = 0.45 \textwidth]{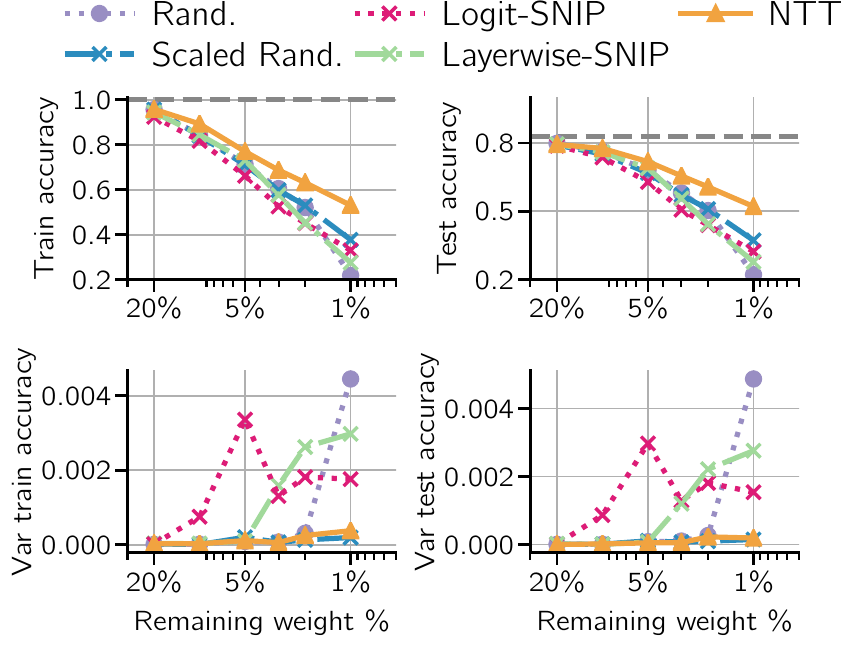}
    \caption{\textbf{Supervised performance on CIFAR-10 using layerwise pruned Conv-4.}  \label{fig:overall_cifar_cnn_conv4}}
\vspace{-0.1 in}
\end{figure}

We found qualitatively similar behavior as in the above experiments: NTT achieved higher endpoint accuracy and reduced variance (Fig.~\ref{fig:overall_cifar_cnn_conv4}) over baseline methods. %

\begin{table*}[t]
\centering
\caption{\footnotesize Supervised performance on CIFAR-10 and SVHN using layerwise pruned Conv-4.}
\label{table:cifar_svhn_results}
\resizebox{0.9\textwidth}{!}{
\begin{tabular}{lcccccc} 
\cmidrule[\heavyrulewidth](lr){1-7} 
\textbf{Dataset} & \multicolumn{3}{c}{CIFAR-10} & \multicolumn{3}{c}{SVHN}
\\
\cmidrule(lr){1-7}
\makecell{Remaining\\weights \%}   & 5\%     & 3\% & 1\%     
     & 5\%     & 3\% & 1\%      \\ 

\cmidrule(lr){1-1}
\cmidrule(lr){2-4}
\cmidrule(lr){5-7}
     
Rand.
& 67.94 $\pm 0.92 $  & 58.09  $\pm 1.02 $  & 22.17 $\pm 6.98 $
& 77.87 $\pm 0.28 $ & 72.09 $\pm 0.66 $ & 18.70  $\pm 1.47 $	
\\

Scaled Rand.
& 66.86 $\pm 1.05 $ & 57.17 $\pm 0.84 $  & 37.24 $\pm 1.22 $    
&  79.55 $\pm 0.26  $ & 75.22 $\pm 0.27 $ & 48.47 $\pm 3.10 $
\\

Logit-SNIP
& 62.93  $\pm 5.45 $ & 50.77 $\pm 3.59 $ & 33.30  $\pm 3.92 $
& 74.89 $\pm 4.40  $ & 66.21 $\pm 4.68 $ & 41.92 $\pm 5.90 $
\\

Layerwise-SNIP
& 68.90 $\pm 0.87 $ & 55.64 $\pm 3.47$  & 27.84 $\pm 5.25$  
& 75.62 $\pm 3.28 $ & 64.08 $\pm 4.77 $ &  43.58 $\pm 3.30  $
\\

NTT
& \bf 71.84 $\pm$ 0.72 & \bf 65.58 $\pm$ 0.76 & \bf 52.14  $\pm$ 1.39
& \bf 81.49 $\pm $ 0.45 & \bf 79.04 $\pm$ 0.32 & \bf 71.63  $\pm$ 0.80
\\ 
\cmidrule[\heavyrulewidth](lr){1-7}
\end{tabular}
}
\vspace{-0.2cm}
\end{table*}

\paragraph{SVHN.}
Using Conv-4, we evaluated the performance of NTT on the larger SVHN dataset (Table \ref{table:cifar_svhn_results}). Consistent with our previous findings, we found NTT initialized networks outperformed other approaches at all sparsity levels.

\subsection{Evaluating globally sparse NTT initialization \label{subsec:global}}
In the experiments reported above, we have focused on layerwise pruning. Compared to the layerwise pruning setting, global pruning typically results in higher accuracy, but with a lower gain in theoretical speedup in CNNs \citep{Blalock2020what}. To investigate the performance of NTT in both settings, we applied NTT to produce globally sparse Conv-4 on CIFAR-10. The results are summarized in Table \ref{table:layerwise_vs_global}. 
These results confirmed our expectations that all methods achieved higher performance in the global setting than in the corresponding layerwise setting. Also, we observed that NTT initialized networks offer a small but consistent improvement over other methods in the global setting. Admittedly, the improvement is lower than in the layerwise case.

\begin{table}[htbp]
\centering
\caption{\footnotesize Supervised performance on CIFAR-10 using globally pruned Conv-4.}
\label{table:layerwise_vs_global}
\resizebox{0.45\textwidth}{!}{
\begin{tabular}{lcccccc} 
\cmidrule[\heavyrulewidth](lr){1-4} 
\makecell{Remaining\\weights \%}     
     & 5\%     & 3\% & 1\%      \\ 

\cmidrule(lr){1-1}
\cmidrule(lr){2-4}

Logit-SNIP
& 77.27	$\pm 0.36 $ & 75.13 $\pm 0.66 $ &	65.90  $\pm 0.48 $
\\

SNIP
& 78.28	 $\pm 0.33 $ & 76.00 $\pm 0.48 $ &	67.12  $\pm 0.46 $
\\

NTT
& \bf 78.85  $\pm$ 0.36  & \bf	76.28 $\pm$ 0.32   & \bf	68.81  $\pm$ 0.28  
\\ 
\cmidrule[\heavyrulewidth](lr){1-4}
\end{tabular}
}
\vspace{-0.2cm}
\end{table}
We next assessed the accuracy-efficiency tradeoff imposed by layerwise and global pruning. First, we numerically confirmed that global pruning tends to preserve orders of magnitude more parameters in the convolutional layers than layerwise pruning at a fixed sparsity level (top panel, Fig.~\ref{fig:compare_layerwise_label_dependent_cifar_cnn_conv4}). Since the computational cost of CNNs is dominated by these convolutional layers \citep{Li2017Pruning}, global pruning leads to a much lower speedup than lawerwise pruning (bottom panel, Fig.~\ref{fig:compare_layerwise_label_dependent_cifar_cnn_conv4}). Thus NTT can be used to exploit this tradeoff between accuracy and efficiency as it excels at both layerwise and global pruning.

\begin{figure}[h] 
  \center
    \includegraphics[width = 0.45 \textwidth]{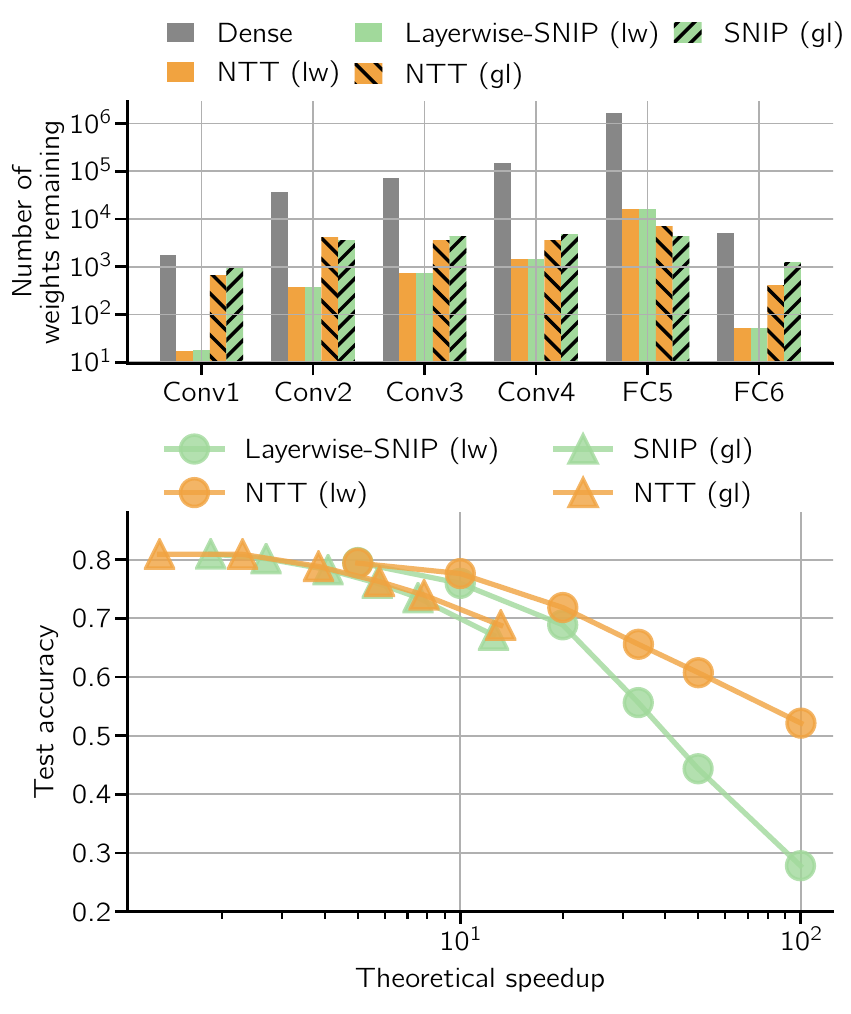}
    \vspace{-0.2 in}
    \caption{\textbf{Tradeoff between model quality and efficiency.} Layerwise and global pruning methods are respectively indicated by lw and gl in figure legends. Top panel: The number of remaining weights in each layer of 1\% sparse Conv-4 networks on CIFAR-10. Bottom panel: Accuracy for Conv-4 on CIFAR-10 for several sparsity levels and their corresponding theoretical speedups. The theoretical speedup is defined as the number of multiply-adds used by the dense model divided by the number used by the sparse model \citep{Blalock2020what}. \label{fig:compare_layerwise_label_dependent_cifar_cnn_conv4}}
    \vspace{-0.2 in}
\end{figure}

\section{Discussion}

We have introduced NTT, a foresight pruning method that finds trainable sparse neural networks capable of learning rapidly and generalizing well in subsequent supervised learning tasks. 
To that end, our method draws on theoretical insights from linearized neural networks and \glspl{ntk} to derive sparse neural networks whose training evolution in network output space remains close to the path of a corresponding dense neural network.
The two key advantages of our method are that (i)~it only requires label-free data, and (ii)~it can be used to find trainable layerwise sparse networks, e.g., \glspl{cnn} with sparse convolutional filters, which are desirable for energy-efficient inference.%

Our work was inspired by the \gls{lth} \citep{Frankle2019} in that we aim to identify ``winning tickets,'' i.e., sparse networks that can be trained to high accuracy. However, our approach is notably different because it does not require supervised training of a dense network
that is pruned subsequently. 
Instead, our approach transfers the \emph{training dynamics} of an \emph{untrained} dense network onto a sparse network at the time of its initialization.
This idea is partially inspired by ``reservoir transfer'' \citep{He2019Reservoir} from a neuromorphic computing context. 
Our work is related to other foresight pruning methods that seek to ensure unimpeded error gradients \citep{Lee2018snip,Wang2020picking}. 
Notably, both \citep{Wang2020picking} and our work involve the \gls{ntk}, but in different contexts. While NTT relies on the \gls{ntk} to constrain the student networks' output evolution under subsequent supervised learning, 
\citet{Wang2020picking} invoked the NTK to study the networks' error gradients during supervised learning. As a result, their pruning criterion depends on a supervised loss function and thus on labeled data.
NTT, on the other hand, is independent of the choice of the loss function and thus inherently label-free.

In biology, the sparseness of both connectivity and activity is thought to be a key design principle for energy-efficient information processing \citep{Sterling2017principles}.
For similar reasons, sparse neural networks are increasingly moving into the focus in machine learning \citep{gong_compressing_2014,lebedev_speeding-up_2015,jaderberg_speeding_2014}
and neuromorphic engineers \citep{Neftci2019,Roy2019towards}. %
In all likelihood, neurobiology has evolved efficient algorithms implemented in its organic wetware to initialize its neural networks in an intelligent and energy-efficient way 
which may encode inductive biases and promote learning  
\citep{Zador2019,richards_deep_2019}.

One of the limitations of the present work is that the calculation of NTK matrices is itself computationally expensive. Therefore we are particularly excited about the possibility of exploring novel, more efficient ways to compute the NTK in the future (e.g., \citet{Li2019enhanced}). 
Another limitation is that we have only considered vision tasks and feed-forward networks in this article. It remains future work to explore NTT in other domains and architectures, e.g., temporal signal processing tasks.

Studies on efficient initialization techniques for sparse neural networks are still at their dawn, and profound theoretical insights are missing.
In the foreseeable future, when the development of dedicated hardware will translate the hypothetical gains of sparse networks into significant performance gains, effective initialization schemes will become indispensable. 
With the introduction of NTT, we have made one step in this direction.  This work opens up exciting new vistas for future work on elucidating the operating principles of computationally efficient sparse neural networks

\section*{Acknowledgements}
This work was supported by the Novartis Research Foundation. We thank the anonymous reviewers for their valuable input. 
Further thanks go to David Belius, Ivan Dokmani\'c, and Yue (Kris) Wu for helpful discussions. 
Finally, we would like to thank the JAX and neural-tangents teams for developing libraries that make NTK computation fast and straightforward.

\bibliography{ntt.bib}

\bibliographystyle{icml2020}

\clearpage
\appendix
\onecolumn
\section*{Appendix}
\section{Proof of Proposition~\ref{prop:linear-opt} }  \label{proof-to-proposition}
Here we prove Proposition~\ref{prop:linear-opt} in the main text.

\begin{prop} 
Consider a linear NTT teacher model $f(\xb, \ab(0)) = \ab(0)^\top \xb$, where $\xb$ and $\ab(0) \in \RR^d$ are model input and initial parameters. Suppose that we are given a linear NTT student model $g_{\mb}(\xb, \widetilde{\ab}(0)) = (\mb \odot \widetilde{\ab}(0))^{\top} \xb$ whose initial parameters $\widetilde{\ab}(0)$ are NTT-optimal in the sense that $J_{\ab(0)}(\mb \odot \widetilde{\ab} (0)) = 0$. Then upon continuous-time and quadratic-loss based gradient descent training, the dense and sparse models' outputs evolve in the same way:
\[ f(\xb, \ab(t)) = g_{\mb}(\xb, \widetilde{\ab} (t)), \]
for all training inputs $\xb$ and time steps $t \geq 0$.
\end{prop}

\begin{proof}[Proof of Proposition~\ref{prop:linear-opt}]

Let $\{\xb_i, y_i\}_{i = 1}^n \subset \RR^d \times \RR$ be a training dataset of $n$~input-target pairs. We first consider the dense linear model $f(\xb, \ab) = \ab^\top \xb$. We use the shorthand notation $\bm{X} = [\xb_i]_{i \in [n]} \in \RR^{d \times n}$ as the column-wise concatenation of training inputs, $f(\bm{X}, \ab )=  \bm{X}^\top \ab \in \RR^n$ as the vector whose entries are outputs of the dense linear model, and $\yb = (y_i)_{i \in [n]} \in \RR^n$ as the corresponding targets. As a special case of  \citep[Lemma 3.1]{Arora2019On}, namely when the model is assumed to be linear, the model output follows the evolution

\begin{equation} \label{eq:linear_model_ode}
\frac{\d f(\bm{X}, \ab (t) ) }{ \d t}  = - \Hb  \Big [ f \Big ( \bm{X}, \ab (t) \Big ) - \yb \Big ], 
\end{equation}
where $\Hb = \bm{X}^\top \bm{X} \in \RR^{n \times n}$. The solution of the linear ODE  \eqref{eq:linear_model_ode} is given by
\begin{equation} \label{eq:linear_evolution_dense} 
f \Big (\bm{X}, \ab (t) \Big ) = e^{- t \Hb} f \Big (\bm{X}, \ab(0) \Big ) + \left [ \Ib - e^{-t \Hb} \right] \yb.
\end{equation}

We now turn to the case of the sparse linear model $g_{\mb}(\widetilde{\ab}, \xb) = (\mb \odot \widetilde{\ab})^{\top} \xb$. For convenience, we write $g_{\mb} \Big (\bm{X}, \ab (t) \Big ) = \bm{X}^\top \text{diag}(\mb) \ab \in \RR^n$ as a vector whose entries are outputs of the sparse linear model, where $\text{diag}(\cdot)$ transforms the $d$-dimensional vector $\mb$ into a $d$-by-$d$ diagonal matrix. Similar to the case of the dense model, the sparse model's output dynamics are characterized by the linear ODE

\begin{equation} \label{eq:linear_sparse_model_ode}
\frac{\d g_{\mb} \big (\bm{X}, \ab (t) \big ) }{ \d t}  = - \widetilde{\Hb}  \Big [ g_{\mb} \Big ( \bm{X}, \ab (t) \Big ) - \yb \Big ], 
\end{equation}
where $\widetilde{\Hb} = (\text{diag}(\mb) \bm{X})^\top \text{diag}(\mb) \bm{X} \in \RR^{n \times n}$. Solving the linear ODE in Eqn. \eqref{eq:linear_sparse_model_ode}, we get
\begin{equation} \label{eq:linear_evolution_sparse} 
g_{\mb} \Big (\bm{X}, \ab (t) \Big ) = e^{- t \widetilde{\Hb}} g_{\mb} \Big (\bm{X}, \widetilde{\ab} (0) \Big ) + [ \Ib - e^{-t \widetilde{\Hb}}] \yb.
\end{equation}

Note that the NTT objective (Eqn. \eqref{eq:ntt_objective}) achieves 0 only if $\Hb = \widetilde{\Hb}$ and $ f\Big (\bm{X}, \ab (0) \Big )=  g_{\mb} \Big (\bm{X}, \widetilde{\ab} (0) \Big ) $. Comparing Eqn. \eqref{eq:linear_evolution_dense} and Eqn. \eqref{eq:linear_evolution_sparse}, we see 

\[ f(\xb, \ab(t)) = g_{\mb} (\xb, \widetilde{\ab} (t)), \]
for all $\xb \in \{\xb_i\}_{i = 1}^n$ and timesteps $t \geq 0$ as claimed. 
\end{proof}

\section{Details of experiment setup \label{supp:experiment_setup}}

We proceed to introduce the details of experiments reported in the main text, including the model architecture, optimization hyper-parameters, and baseline methods. All experiments were performed using JAX \citep{Jax2018github} and the neural-tangents library \citep{neuraltangents2020}. \\

\subsection{Neural network architecture \label{supp:conv_4}}

For most of the MNIST and Fashion MNIST classification tasks, we use the standard Lenet-300-100 MLP and Lenet-5-caffe CNN architecture together with Relu activations for hidden layers and softmax cross-entropy loss on logit outputs. One exception is the toy example reported in Sec. \ref{subsec:probing}, where we used 2~linear output neurons to perform regression. 

For the CIFAR-10 and SVHN datasets, we used a CNN model consisting of 4~convolution layers followed by 2 feedforward layers with dropout (Table \ref{tab:keras_net}). This can be considered as a slightly scaled-up version of the CNN from the Keras tutorial\footnote{\url{https://keras.io/examples/cifar10_cnn/}}, in which the only modification we made was to double the number of filters in each convolutional layer. 

\begin{table}[ht]
\centering
\footnotesize
\begin{tabular}{lccccc} \toprule
Operation     & Filter size   & \# Filters & Stride         &   Dropout rate  & Activation \\ \midrule
3x32x32 input &   --          & --         & --             &   --   &  --                     \\ 
Conv          & $3 \times 3$  & 64         & $1 \times 1$   &   --   &  ReLu                   \\ 
Conv          & $3 \times 3$  & 64         & $1 \times 1$   &   --   &  ReLu                   \\ 
MaxPool       & --            & --         & $2 \times 2$   &  0.25     &  --                  \\ 
Conv          & $3 \times 3$  & 128        & $1 \times 1$   &   --   &  ReLu                  \\ 
Conv          & $3 \times 3$  & 128        & $1 \times 1$   &   --   &  ReLu                  \\ 
MaxPool       & --            & --         & $2 \times 2$   &  0.25    &  --                \\ 
FC            & --            & 512        & --             &   0.5   &  ReLu              \\ 
FC            & --            & 10         & --             &   --    &  Softmax                  \\  \bottomrule
\end{tabular}
\caption{The Conv-4 architecture used for the CIFAR-10 and SVHN tasks. The dropout rate is defined to be the fraction of the input units to drop.}
\label{tab:keras_net}
\end{table}%

\subsection{NTT hyperparameters \label{supp:ntt_experiment}}
In Table \ref{tab:ntt_params} we summarize the hyperparameters used in the NTT optimization stage. 
For each experiment, we initialized NTT teachers using the Glorot initialization scheme \citep{Glorot10}; we then perform perform gradient-based optimization using the Adam optimizer \citep{Kingma2015adam} with various learning rates and batch sizes (see Table \ref{tab:ntt_params}). \\

\begin{table}[ht]
\centering
\footnotesize
\scalebox{0.85}{
\begin{tabular}{llcccccc} \toprule
Task                                & Model         & \shortstack{Epoch \\number} & \shortstack{Batch \\ size} & \shortstack{Learning \\ rate $\eta$}  & \shortstack[l]{Mask update \\ frequency} &  \shortstack{Strength \\ parameter $\gamma^2$}  & \shortstack{weight-decay \\ constant $\beta$}\\ \midrule
\multirow{2}{*}{Visualization}  & Lenet-300-100          & 5000 (full-batch)          & 500        & 1e-03   &   100 iters     &  1e-5       & 0             \\ 
                                & Lenet-5-caffe         & 5000 (full-batch)           & 500         & 5e-04   &   100 iters     &  1e-6      & 0              \\ \midrule
\multirow{2}{*}{MNIST}          & Lenet-300-100 & 20                          & 64          & 5e-04   &   100 iters     &  1e-3         & 1e-4          \\ 
                                & Lenet-5-caffe & 20                          & 64          & 5e-04   &   100 iters     &  1e-3         & 1e-5          \\ \midrule
\multirow{2}{*}{Fashion-MNIST}  & Lenet-300-100 & 20                          & 64          & 5e-04   &   100 iters     & 1e-3           & 1e-4         \\ 
                                & Lenet-5-caffe & 20                          & 64          & 5e-04   &   100 iters     &  1e-3           & 1e-5       \\ \midrule
CIFAR-10                        & Conv-4 CNN (see Table \ref{tab:keras_net})     & 20                          & 32          & 5e-04   &   100 iters     &   1e-3    & 1e-8              \\\midrule
SVHN                        & Conv-4 CNN (see Table \ref{tab:keras_net})     & 20                          & 32          & 5e-04   &   100 iters     &   1e-3    & 1e-8              \\
\bottomrule
\end{tabular}}
\caption{Hyperparameters used for NTT in this paper. }
\label{tab:ntt_params}
\end{table}\normalsize

The toy example in Section~\ref{subsec:probing} of the main text deserves some additional comments. For this task, we used 500 images from the MNIST dataset, containing 250 images of each digit 0 and 1. 
We performed 5000~iterations of full-batch gradient descent for this task; for this reason, 5000 is also the total number of epochs.

\subsection{Supervised learning hyperparameters}
Regarding the supervised learning experiments, we spared 10\% of the training data for model validation purposes and only used 90\% for model training. We used the Adam optimizer with learning rate 1e-3, $\beta_1 = 0.9$, and $\beta_2 = 0.999$ for all supervised learning tasks except for the visualization task in Sec.~\ref{subsec:probing}, in which the stochastic gradient descent optimizer with learning rate~0.01 was used. 
In addition, all experiments, except for the visualization task, used a minibatch-size of 64. 
For MNIST and Fashion MNIST experiments, we performed optimization for 50 epochs. On CIFAR-10, we trained for 600 epochs. 

\subsection{Baseline pruning methods \label{supp:snip_logit_snip}}
In this subsection, we first recap the SNIP pruning method \citep{Lee2018snip} and introduce two straightforward extensions of it, Layerwise-SNIP and Logit-SNIP, which were used as baselines for NTT. Finally, we point out some technicalities of random pruning baselines.

\paragraph{SNIP and Layerwise-SNIP}
 Recall that SNIP \citep{Lee2018snip} assigns each neural network parameter $\theta$ a sensitivity score $S(\theta)$ defined as 

\[S(\theta) = \left| \theta \cdot \frac{\partial L_{\thetab} }{\partial \theta} \right|,\]
where $L_{\thetab} = \sum_{i = 1}^{n_B} L( f(\xb_i, \thetab), \yb_i)$ is the loss evaluated over a batch of $n_B$ number of input-output data pairs $\{\xb_i, \yb_i\}_{i = 1}^{n_B}$ and $\thetab$ is the vector of randomly initialized parameters. \citet{Lee2018snip} proposed to remove neural network parameters with lowest sensitivity scores. That is, in its original formulation, SNIP is a global pruning method.
To be consistent with \citep{Lee2018snip}, we reserve the terminology SNIP to only be used in global pruning context.
A straightforward way to turn SNIP into a layerwise pruning method is to remove a fixed fraction of the parameters having the lowest sensitivity scores from each layer. We call this extension Layerwise-SNIP. 

\paragraph{Logit-SNIP and layerwise Logit-SNIP}
The SNIP and Layerwise-SNIP methods described above depend on labels. Here we provide a label-free extension: We modify the sensitivity score $S(\theta)$ into a logit-based sensitivity score $\tilde{S}(\theta)$ defined as 

\[\tilde{S}(\theta) = \left| \theta \cdot \frac{\partial Z_{\thetab} }{\partial \theta} \right|,\]
where $Z_{\thetab} = \sum_{i = 1}^{n_B} \|f(\xb_i, \thetab)\|_2^2$. We can perform either global or layerwise pruning in reference to the scores $\tilde{S}(\theta)$. We refer the global pruning criteria as Logit-SNIP and layerwise criteria as layerwise Logit-SNIP. When the context is clear, we may use the terminology Logit-SNIP to refer to either its layerwise or global variant.

\paragraph{Random pruning}
In the main text we have explained two ways to randomly sample sparse neural networks. Note that for these random methods, their global pruning variant is equivalent to their respective layerwise variant: In either formulation, each weight parameter receives an identical chance to be removed and therefore the expected fraction of pruned parameters for each layer is the same. In each run, we randomly remove such expected fraction of parameters from each layer.

\section{Additional experiments and results}

\subsection{Visualizing network output evolution in CNNs \label{supp:visualization}}
We repeated the experiment introduced in Sec.~\ref{subsec:probing} of the main text using a Lenet-5-like architecture with two linear output neurons (Fig.~\ref{fig:cnn_probing}). In good agreement with the MLP results, we found that the NTT teacher and student follow a similar output evolution during supervised training.

\begin{figure}[!ht] 
    \center
    \includegraphics[width = 0.6 \textwidth]{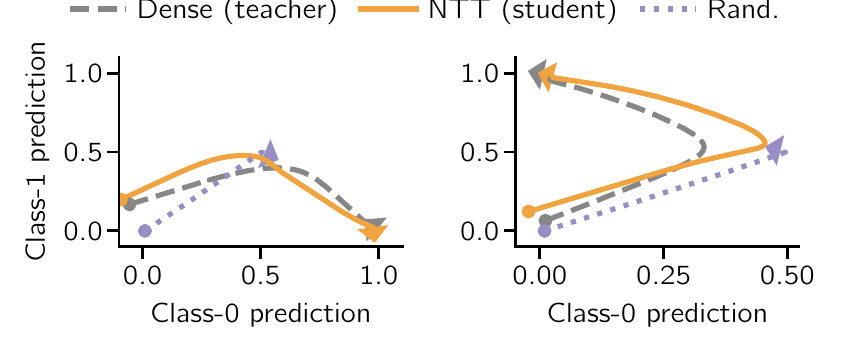}
    \caption{\textbf{A CNN NTT student's output evolution closely follows its dense teacher's.} \label{fig:cnn_probing}}
\end{figure}

\subsection{Experiments on global pruning \label{supp:global_pruning}}

In the main text, we have focused on layerwise pruning. Here we provide experimental results on global pruning methods using various datasets. The overall experiment procedure and hyperparameters used in this set global pruning experiments are identical to the settings in the layerwise experiments, except at one place: During NTT, instead of initializing the binary mask based on weight magnitudes as outlined in Sec \ref{subsection:ntt_implementation}, we found that the NTT optimization process converges slightly faster if we use the Logit-SNIP produced mask as the initialization. Since Logit-SNIP masks do not depend on labels, the NTT procedure remains label-independent. 

\begin{figure}[ht]
\centering     %
\subfigure[Lenet-300-100 on MNIST]{\includegraphics[width=0.45\textwidth]{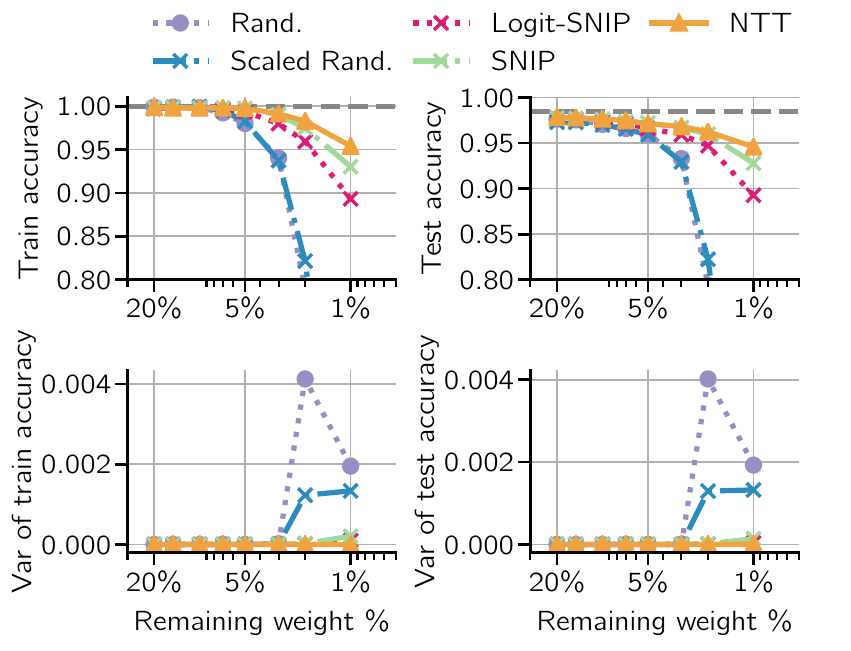}}
~
\subfigure[Lenet-5-caffe on MNIST]{\includegraphics[width=0.45\textwidth]{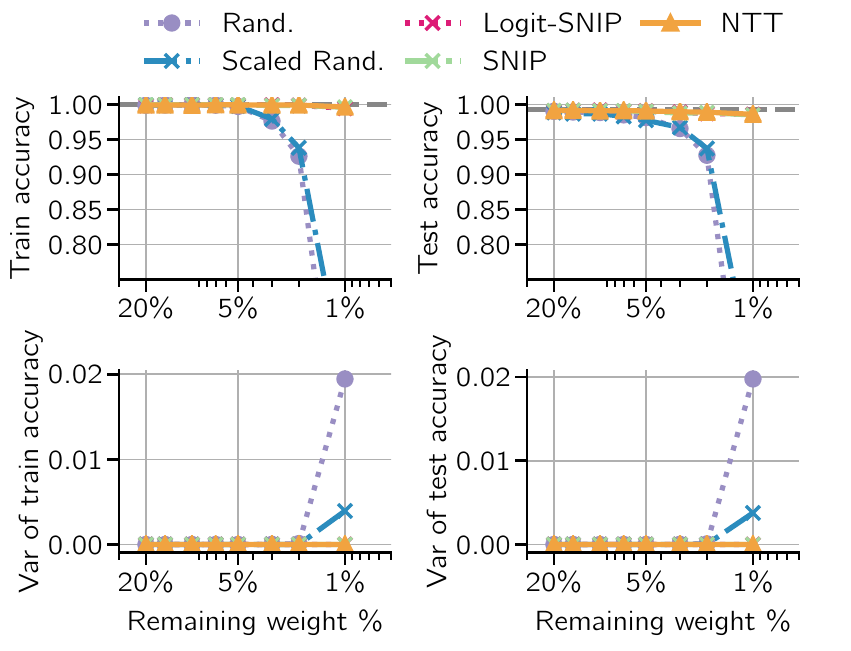}}
\subfigure[Lenet-300-100 on Fashion MNIST]{\includegraphics[width=0.45\textwidth]{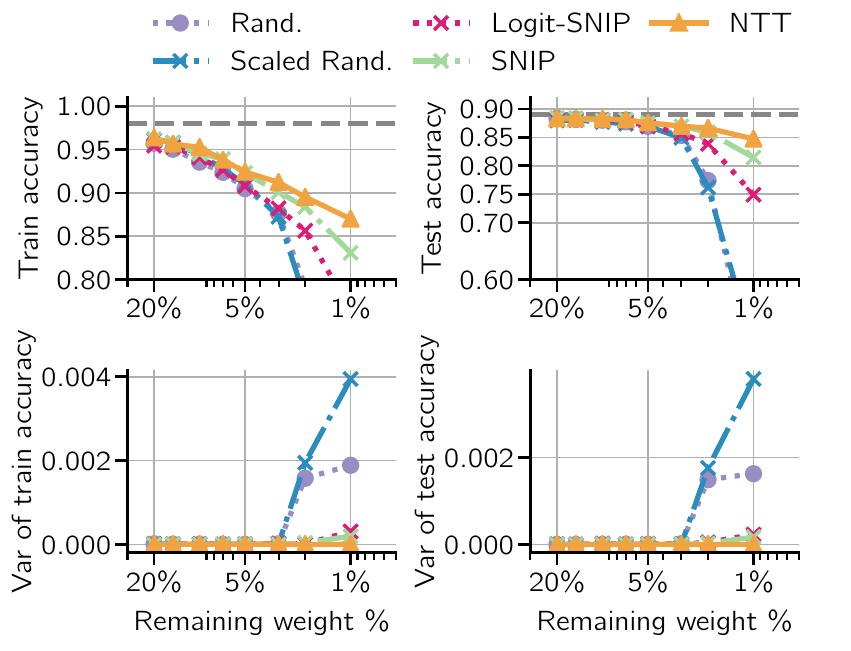}}
~
\subfigure[Lenet-5-caffe on Fashion  MNIST]{\includegraphics[width=0.45\textwidth]{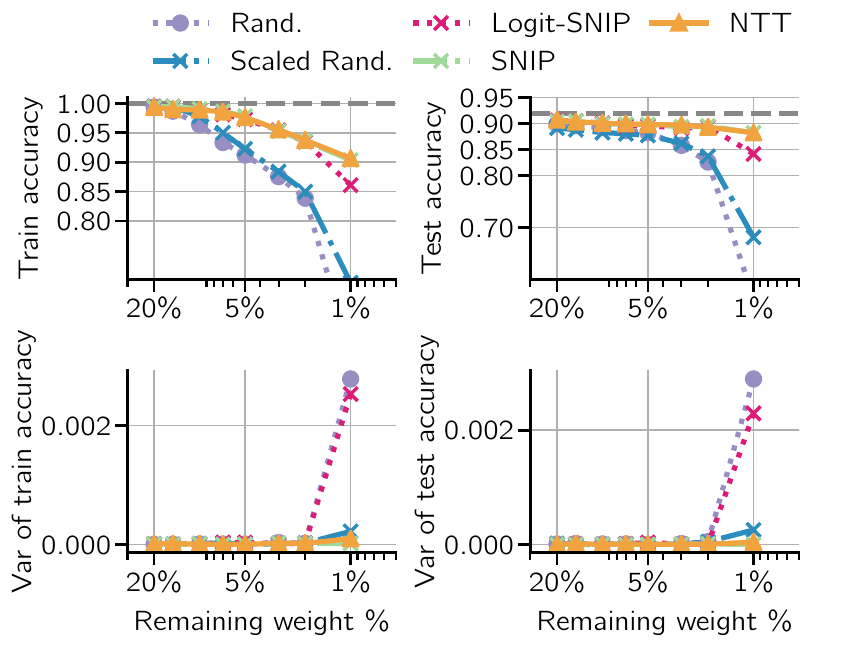}}

\subfigure[Conv-4 on CIFAR-10]{\includegraphics[width=0.45\textwidth]{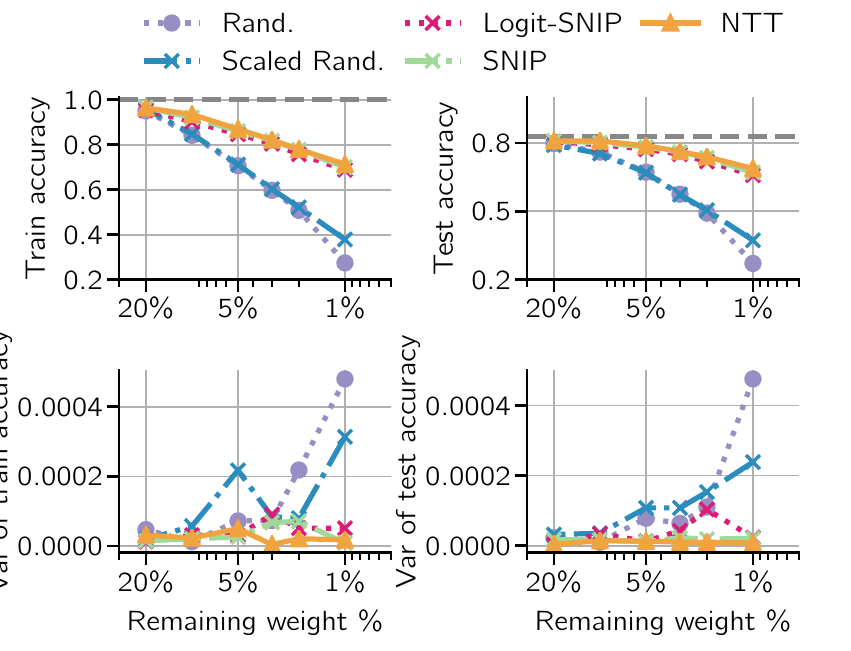}}
\caption{Supervised performance of NTT and baseline methods under global pruning.}

\end{figure}

\end{document}